%% file: main.tex
\documentclass{article}
\usepackage{arxiv}
\usepackage{mycommands}
\usepackage[utf8]{inputenc} 
\usepackage[T1]{fontenc}    
\usepackage{hyperref}       
\usepackage{booktabs}       
\usepackage{amsfonts}       
\usepackage{nicefrac}       
\usepackage{microtype}      
\usepackage{lipsum}		
\usepackage{graphicx}
\usepackage{wrapfig}

\usepackage[backend=biber, style=authoryear-comp,]{biblatex}
\usepackage{sidecap}
\sidecaptionvpos{figure}{t}
\usepackage{subcaption}
\usepackage{color}

\date{\vspace{-5ex}}
\title{On the curvature of the loss landscape}

\date{} 					

\author{{Alison Pouplin} \\
	Technical University of Denmark\\
	\texttt{alpu@dtu.dk} \\
        \And
	{Hrittik Roy} \\
	Technical University of Denmark \\
	\And
	{Sidak Pal Singh} \\
	ETH Zurich, Switzerland \\
        \And
	{Georgios Arvanitidis} \\
	Technical University of Denmark \\
}

\renewcommand{\headeright}{Preprint}
\renewcommand{\shorttitle}{On the curvature of the loss landscape}

\begin{document}

\maketitle
\begin{abstract}
    One of the main challenges in modern deep learning is to understand why such over-parameterized models perform so well when trained on finite data.
    A way to analyze this generalization concept is through the properties of the associated loss landscape.
    In this work, we consider the loss landscape as an embedded Riemannian manifold and show that the differential geometric properties of the manifold can be used when analyzing the generalization abilities of a deep net.
    In particular, we focus on the scalar curvature, which can be computed analytically for our manifold, and show connections to several settings that potentially imply generalization.
\end{abstract}

\section{Flatness and generalization in machine learning}

The relationship between the generalization ability of a model and the flatness of its loss landscape has been a subject of interest in machine learning. \textbf{Flatness} refers to the shape of the hypersurface representing the loss function, parameterized by the parameters of the model. Flat minima are characterized by a wide and shallow basin. \textbf{Generalization} refers to the ability of a model to perform well on unseen data. A widely accepted hypothesis, proposed by various research groups \parencite{hochreiter1997flat, hinton1993keeping, buntine1991bayesian} several decades ago, suggests that flat minima are associated with better generalization compared to sharp minima. The basis of this hypothesis stems from the observation that when the minima of the optimization landscape are flatter, it enables the utilization of weights with lower precision. This, in turn, has the potential to improve the robustness of the model. \\

\begin{minipage}{\textwidth}
  \begin{minipage}{0.6\textwidth}
    \includegraphics[width=\textwidth]{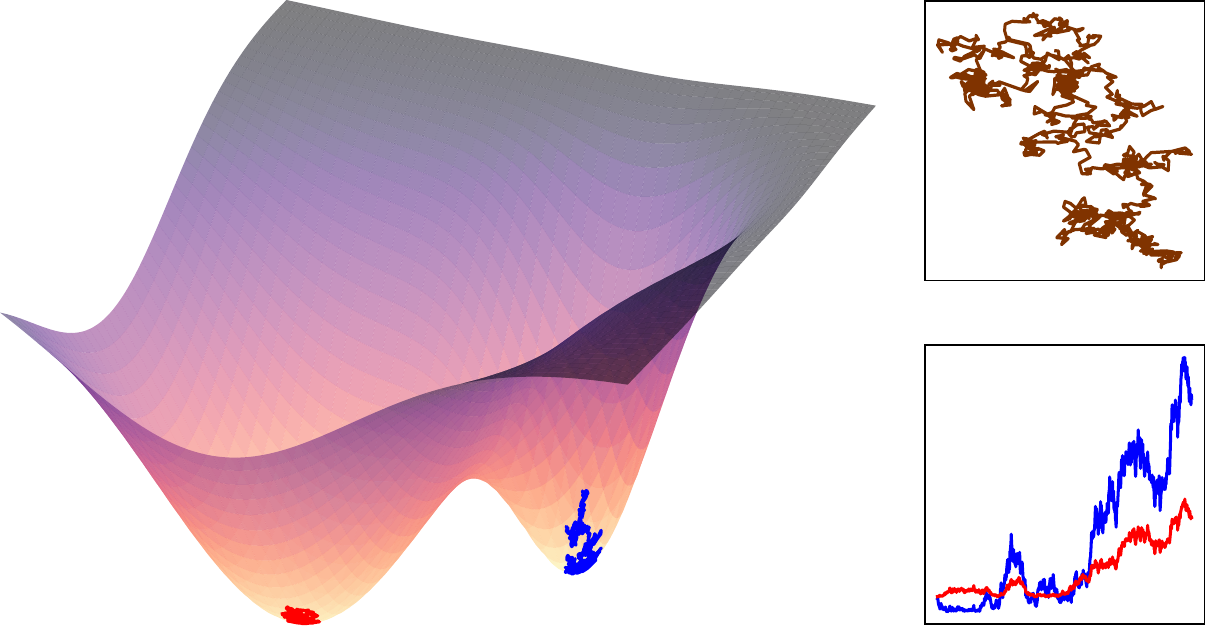}
  \end{minipage}%
  \qquad
  \begin{minipage}{0.4\textwidth}
    \captionsetup{width=\textwidth}
    \captionof{figure}{\textbf{On the left}, a surface represents a loss function $f(\uu, \vv)$ on its parameter space $\{\uu, \vv\}$. We can see two minima, a sharp minima and a flatter minima. A Brownian motion navigates the parameter space around those two minima, in blue for the sharp one, and red for the shallow one. \textbf{On the right}, the upper figure represents the Brownian motion navigating in the parameter space. The same is used for both minima. The lower figure represents the perturbations of the loss $f$ in both the sharp (blue) and flat (red) minima. The loss is more robust to perturbation in the flatter minima.}
  \end{minipage}
\end{minipage}

The notion of flatness has been challenged by \textcite{dinh2017sharp}, who argued that the different flatness measures proposed are not invariant under reparametrization of the parameter space and questioned the assumption that flatness directly causes generalization. \\

Yet, numerous empirical and theoretical studies have presented compelling evidence that supports the relationship between flatness and enhanced generalization. This relationship has been observed in various contexts, by averaging weights \parencite{izmailov2018averaging}, studying inductive biases \parencite{neyshabur2017geometry, imaizumi2022generalization}, introducing different noise in gradient descent \parencite{chaudhari2019entropy, pittorino2021entropic}, adopting smaller batch sizes \parencite{keskar2016large}, and investigating ReLU Neural networks \parencite{yi2019positively}. 

The exact relationship between flatness and generalization is still an open problem in machine learning. In this preliminary work, we build upon the \textit{flatness hypothesis} as a primary motivation to investigate the curvature of the loss landscape, approaching it from a differential geometric perspective. \\

\textbf{In this preliminary work}, we analyze the loss landscape as a Riemannian manifold and derive its \textit{scalar curvature}, an intrinsic Riemannian object that characterizes the local curvature of the manifold. We found that the scalar curvature, at minima, has a straightforward expression and can be related to the norm of the Hessian. While the norm of the Hessian may not always accurately measure flatness, it remains a valuable indicator for understanding optimization. Our findings demonstrate that the scalar curvature possesses all the benefits of the Hessian norm without its limitations.

\section{Geometry of the loss landscape and curvature}

We are interested in finding the parameters $\xx$ of a model that minimizes the loss function denoted $f$. The loss function is a smooth function defined on the parameter space $\MM\subset\RR^{q}$, where $q$ is the number of parameters. In order to study the loss landscape of a model, we can look at the geometry of the graph of the loss function, which is a hypersurface embedded in $\RR^{q+1}$.

\begin{definition}[\textbf{Metric of a graph}]
    Let $f:\Omega\subset\RR^{q} \to \RR$ be a smooth function. We call \textbf{graph of a function} the set: 
    \[\Gamma_f=\{(\xx, y) \in \Omega \times \RR  \mid y = f(\xx)\}.\]
    
    The graph $\Gamma_f$ is an topological smooth manifold embedded in $\R^{q+1}$, and it is isometric to the Riemannian manifold $(\MM, g)$ with $\MM\subset\RR^{q}$ and the induced metric 
    \begin{equation}
    \label{eq:riemannian-metric}
        g_{ij} = \delta_{ij} + \partial_i f \partial_j f.
    \end{equation}
\end{definition}

the metric is obtained by pulling back, in one case, the loss function to the parameter space ($\partial_i f \partial_j f$), and in another case, the parameter space to itself ($\delta_{ij}$), \parencite{lee2018introduction}. 

Instead of working in the ambient space $\RR^{q+1}$, it is more convenient to study the intrinsic geometry of the loss function in the parameter space ($\MM$, $g$). In particular, knowing the Riemannian metric, we can compute the associated geometric quantities of the loss landscape as the Christoffel symbols, the Riemannian curvature tensor, and the scalar curvature (See Appendix~\ref{appendix:definition} for an introduction of those quantities).  In the following, we will denote $\nabla$ the Euclidean gradient operator of the loss function $f$, and $\Hess$ the Euclidean Hessian of $f$. 

\begin{tabular}{l l}
    Gradient($f$): & $(\nabla f)_{i} = \J_i= \partial_i f = f_{,i}$\\
    Hessian($f$): & $(\Hess)_{ij} = \partial_i \partial_j f = f_{,ij}$
\end{tabular}

\subsection*{Curvature in Riemannian geometry}
The Christoffel symbols define a corrective term used to compute covariant derivatives in a curved space. They can be derived from the Riemannian metric. 
\begin{proposition}[\textbf{Christoffel symbols}]
The Christoffel symbols are given by:  
\begin{equation}
    \Gamma^{i}_{kl} = \beta f_{,i} f_{,kl},
\end{equation}
with $\beta = (1+\norm{\nabla f}^2)^{-1}$.
\end{proposition}
\begin{proof}
See Appendix~\ref{appendix:proofs}.
\end{proof}

Using those Christoffel symbols, we can directly compute the Riemannian curvature tensor. Using the Einstein summation convention, the Riemannian curvature tensor is an intrinsic mathematical object that characterizes the deviation of the curved manifold from the flat Euclidean manifold.
\begin{proposition}[\textbf{Riemannian curvature tensor}]
    The Riemannian curvature tensor is given by: 
    \begin{equation}
        \R^i_{jkm} = \beta (f_{,ik}f_{,jm} - f_{,jm}f_{,jk}) - \beta^2 f_{,i}f_{,r} (f_{,rk}f_{,im} - f_{,rm}f_{,jk}),
    \end{equation}
    with $\beta = (1+\norm{\nabla f}^2)^{-1}$.
\end{proposition}
\begin{proof}
    See Appendix~\ref{appendix:proofs}.
\end{proof}

While those four-dimensional tensor gives us a complete picture of the curvature of a manifold, it can be difficult to interpret in practice. Instead, a scalar object, the scalar curvature, can be derived from the Riemannian curvature tensor. The scalar curvature quantifies locally how curved is the manifold.  
\begin{proposition}[\textbf{Scalar curvature}] \label{proposition:full_scalar_curvature}
    The scalar curvature is given by: 
    \begin{equation}
        \Sc =  \beta \left(\trace(\Hess)^2 - \trace(\Hess^2)\right) + 2 \beta^2 \left( \nabla f^{\top} (\Hess^2 - \trace(\Hess) \Hess) \nabla f\right),
    \end{equation}
    with $\beta = (1+\norm{\nabla f}^2)^{-1}$.
\end{proposition}
\begin{proof}
    See Appendix~\ref{appendix:proofs}.
\end{proof}

This expression simplifies when the gradient is zero, which corresponds to a critical point of the loss function.
In this case, the scalar curvature is given by:
\begin{corollary}\label{cor:scalar_curvature}
    When an extremum is reached ($\nabla f=0$), the scalar curvature becomes: 
    \begin{equation}
        \Sc(\xx_{\text{min}}) =  \trace(\Hess)^2 - \trace(\Hess^2)
    \end{equation} 
\end{corollary}
\begin{proof}
    This is a direct result of Proposition~\ref{proposition:full_scalar_curvature}, when $\nabla f = 0$.
\end{proof}

Note that we can also write, at the minimum, $\Sc(\xx_{\text{min}}) = \norm{\Hess}_{*}^2 - \norm{\Hess}_F^2$, with $\norm{\cdot}_{*}$ the nuclear norm and $\norm{\cdot}_F$ the Frobenius norm. 

\subsection*{The scalar curvature as the deviation of the volume of geodesic balls}
This scalar curvature has a simple interpretation, as it corresponds to the difference in volume between a geodesic ball embedded in the Riemannian manifold and a ball of reference, the Euclidean ball. In hyperbolic spaces, the Riemannian ball will be bigger than the Euclidean one, and in spherical spaces, it will be smaller. If the curved space is flat, they are both equal in volume, and the scalar curvature is null.

\begin{proposition}
    \parencite[Theorem 3.98]{gallot1990riemannian} \\
    The scalar curvature $\Sc(\x)$ at a point $\x\in\MM$ of the Riemannian manifold of dimension $q$ is related to the asymptotic expansion of the volume of a ball on the manifold $\mathcal{B}_g(r)$ compared to the volume of the ball in the Euclidean space $\mathcal{B}_e(r)$, when the radius $r$ tends to 0. 
    \[\vol (\mathcal{B}_g(r)) = \vol(\mathcal{B}_e(r)) \left(1-\frac{\Sc(\x)}{6(q+2)} r^2 + o(r^2)\right) \]
\end{proposition}

\section{Scalar curvature and optimization}

Corollary~\ref{cor:scalar_curvature} establishes a connection between the scalar curvature at each peak or valley in the loss landscape and the magnitude of the Hessian: $\Sc(\xx) = \norm{\Hess}{*}^2 - \norm{\Hess}_{F}^2$. Although the Hessian norm plays a key role in optimization tasks, we contend that it is not the most reliable gauge of flatness in \textit{all} situations. On one hand, will delve into some issues that arise from only using the Hessian norm in Section~\ref{sec:limitations_trace}. On the other hand, we will see how the scalar curvature reduces to the Hessian norm in some cases and supports theoretical findings in optimization in Section~\ref{sec:scalarcurvature_optimisation}.

\subsection{Limitations of the trace of the Hessian as a measure of flatness}\label{sec:limitations_trace}

The Hessian of the loss function, specifically its trace, has been shown to influence the convergence of optimization algorithms. For instance, \textcite{wei2019noise} revealed that stochastic gradient descent (SGD) reduces the trace of the loss function's Hessian in the context of over-parameterized networks. In a similar vein, \textcite{orvieto2022anticorrelated} discovered that SGD with anti-correlated perturbations enhances generalization due to the induced noise reducing the Hessian's trace. They also identified that the trace serves as an upper limit on the mean loss over a posterior distribution. Furthermore, within Graphical Neural Networks, \textcite{ju2023generalization} demonstrated that the trace of the Hessian can evaluate the model's resilience to noise. \\

\subsubsection*{The saddle point problem}
Yet, relying solely on the trace of the Hessian may not provide an accurate measure of flatness. For instance, if half of the eigenvalues are positive and the other half are negative, with their sum equaling zero, the trace of the Hessian will also be zero. This is misleading as it suggests a flat region, when in reality it is a saddle point. \\

\begin{example}[Curvature of a parameterized function]
Let us imagine that the loss is represented by a function taking in inputs two weights $u$ and $v$ such that: 
\[f(u,v) = e^{-c u} \sin(u) \sin(v), \]
with $c$ a positive constant. We notably have $\lim_{u\to\infty} f(u,v) = 0$, and so the surface tends to be flatter with $u$ increasing. \\

The trace of the Hessian of $f$ and its scalar curvature can be computed analytically, and we have at a point $\x=(u,v)$: 
\begin{align*}
\trace(\Hess)(\x) = & e^{-c u} (-2 u\cos(u) + (c^2-2)\sin(u)) \sin(v) \\
\Sc(\x) = & \frac{(c^2-1)\cos(2 u)-\cos(2v) - c (c-2\sin(2u))}{e^{2cu} + \cos(v)\sin(u)^2 + (\cos(u)-c\sin(u))\sin(v)^2} &
\end{align*}

\end{example}

\begin{figure}[ht]
    \centering
    \includegraphics[width=\textwidth]{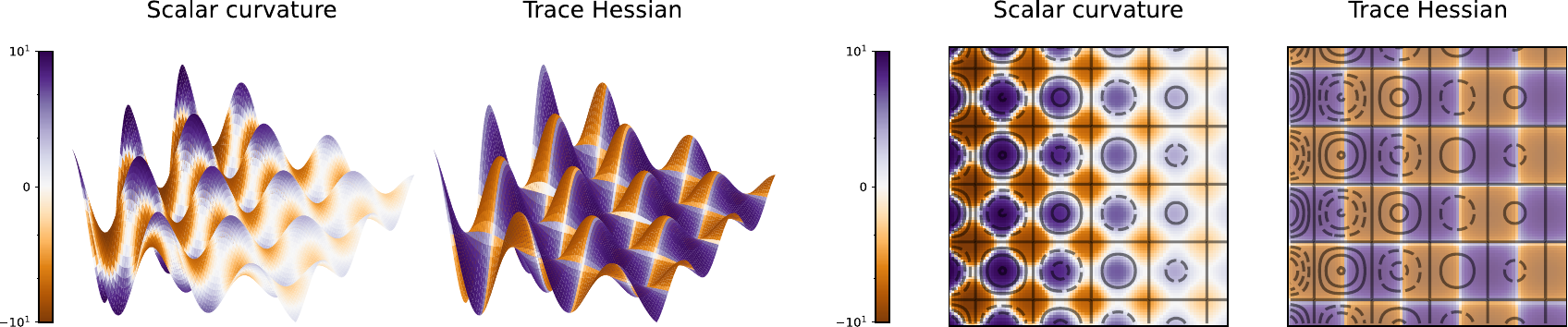}
    \caption{In this figure, the loss function is defined as $f(u,v) = e^{-c u} \sin(u) \sin(v)$, with $c=0.1$. The first two figures represent the surface in 3d, while the two last figures represent the surface seen from above, in the $\{u,v\}$-space. Both the scalar curvature and the trace of the Hessian are shown through the gradient of color.}
    \label{fig:saddlepoint}
\end{figure}

\subsubsection*{The expected flatness over mini-batches}

\begin{wrapfigure}[14]{r}{0.4\textwidth}
\vspace{-20pt}
  \begin{center}
    \includegraphics[width=0.4\textwidth]{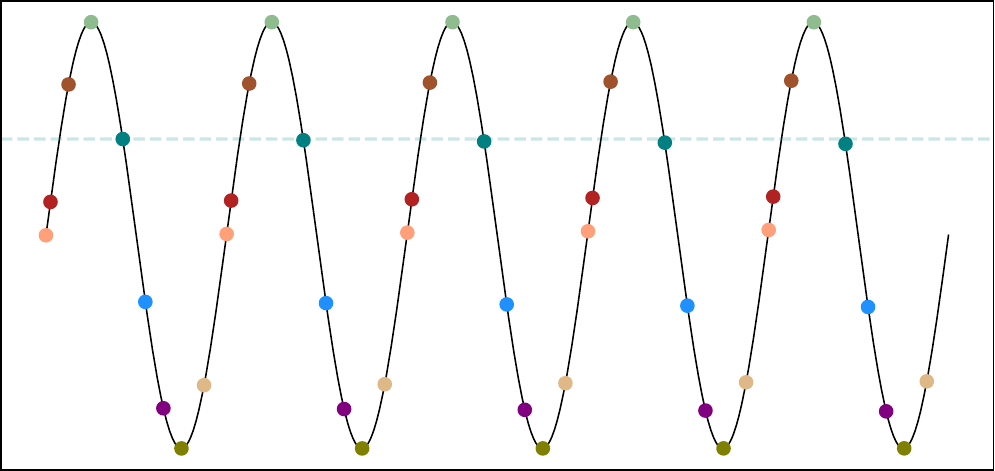}
  \end{center}
  \vspace{-5pt}
  \caption{The data points fit a sinus. The dataset is split into 7 batches of different colors. If the flatness is defined as $\trace(\Hess)$, the flatness over the entire dataset is equal to the expectation of the flatness of a batch. Thus, the curve is considered flat.}
  \label{fig:minibatches}
\end{wrapfigure}

Another challenge emerges when the dataset is divided into small batches. If we choose the Hessian's trace as the measure of flatness, the overall flatness of the entire dataset equals the average flatness over these batches (Equation~\ref{eq:minibatch_trH}). This could potentially induce the wrong conclusion depending on the method used to partition the dataset: In Figure~\ref{fig:minibatches}, the dataset is split in such a way that the trace of the Hessian is null for each batch, which means that the curve is considered as flat over the entire dataset.

The dataset, denoted $\mathcal{D}$, is split into $k$ mini-batches: $\{ \mathcal{B}_1, \mathcal{B}_2, \dots, \mathcal{B}_k \}$. By linearity, the Hessian of the loss function over the entire dataset can be written as the mean of the Hessian of mini-batches i.e.: 
\[\Hess_{\mathcal{D}} = \frac{1}{k}\sum_i \Hess_{\mathcal{B}_i}
\]

As a consequence, since the trace commutes with a summation, we have: $ \trace(\Hess_{\mathcal{D}}) = \trace(\frac{1}{k}\sum_i \Hess_{\mathcal{B}_i}) = \frac{1}{k}\sum_i  \trace(\Hess_{\mathcal{B}_i}) = \mathbb{E}[\trace(\Hess_{\mathcal{B}_i})]$. The trace of the Hessian of the loss function over the entire dataset is the expectation of the Hessian over mini-batches:
\begin{equation} \label{eq:minibatch_trH}
        \trace(\Hess_{\mathcal{D}}) = \mathbb{E}[\trace(\Hess_{\mathcal{B}_i})]
\end{equation}

The corresponding result does not hold for the scalar curvature in general. 
\begin{proposition}
    The scalar curvature of the hessian of the full dataset is not equal to the expectation of the Scalar curvature over mini-batches. That is there exists a dataset, $\mathcal{D}$, and mini-batches, $\{ \mathcal{B}_1, \mathcal{B}_2, \dots, \mathcal{B}_k \}$ such that:
    \[\Sc(\Hess_{\mathcal{D}}) \neq \mathbb{E}[\Sc(\Hess_{\mathcal{B}_i})] \]
\end{proposition}

\begin{proof}
    See Appendix~\ref{appendix:proofs}.
\end{proof}

\subsection{The scalar curvature supports previous theoretical findings through the Hessian norm} \label{sec:scalarcurvature_optimisation}

Although the two previous given examples suggest that in some cases, the trace of the Hessian is not a good definition of flatness, it is associated with the optimization process and the model's capacity to generalize in various ways. We will observe that under certain circumstances, the scalar curvature simplifies to the Hessian norm.

\subsubsection*{Perturbations on the weights}

\textcite{seong2018towards} showed that the robustness of the loss function to inputs perturbations is related to the Hessian. We similarly show that the resilience of the loss function to weights perturbations is upper bounded by the norm of the Hessian. Additionally, a smaller scalar curvature implies stronger robustness.

\begin{proposition}
\label{prop:perturbation-robustness}
Let $\xx_{\text{min}}$ an extremum, $\varepsilon$, a small scalar ($\varepsilon \ll 1$) and $\x$ a normalized vector ($\norm{\xx}=1$). The trace of the square of the Hessian is an upper bound to the difference of the loss functions when perturbed by the weights: 
\begin{equation}
    \norm{f(\xx_{\text{min}} + \varepsilon \xx) - f(\xx_{\text{min}})}_{2}^2 \leq \frac{1}{4} \varepsilon^4 \trace(\Hess^2_{\text{min}})
\end{equation}
\end{proposition}

\begin{proof}
This is obtained by applying the Taylor expansion, for a very small pertubation $\varepsilon  \ll 1$. See Appendix~\ref{appendix:proofs} for the full proof.  
\end{proof}

\begin{SCfigure}
     \centering
     \begin{subfigure}[b]{0.3\textwidth}
         \centering
         \includegraphics[width=\textwidth]{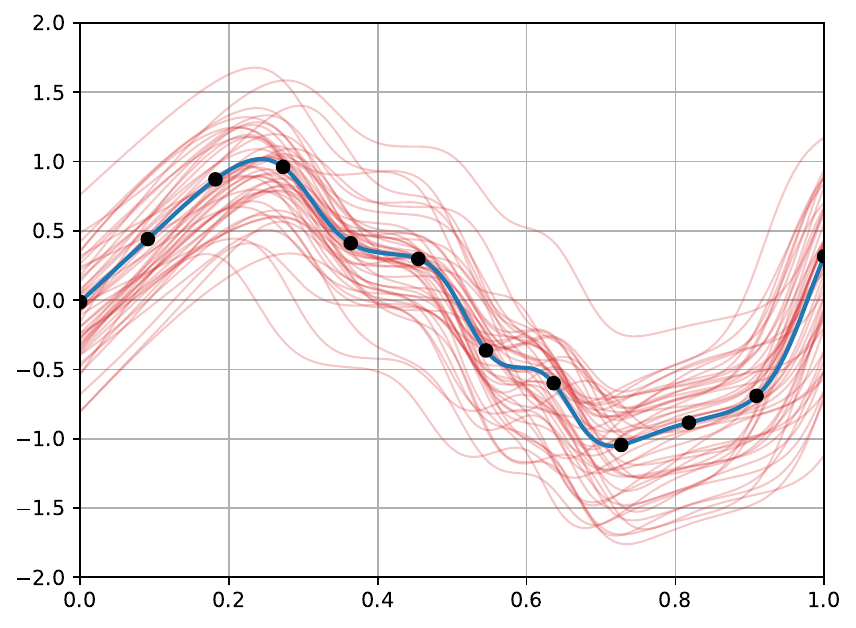}
         \label{fig:y equals x}
     \end{subfigure}
     \hfill
     \begin{subfigure}[b]{0.3\textwidth}
         \centering
         \includegraphics[width=\textwidth]{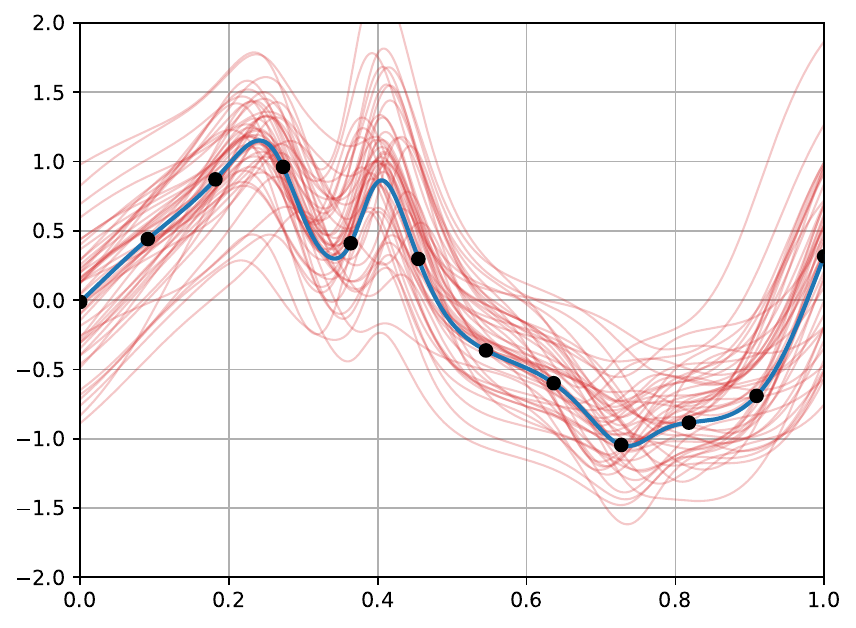}
         \label{fig:three sin x}
     \end{subfigure}
     \caption{Empirical demonstration of Proposition~\ref{prop:perturbation-robustness}. 
     We train two identical and differently initialized deep nets using the same optimizer (Adam).
     We then perturb pointwise the learned weights using Gaussian noise $\mathcal{N}(0,0.1^2)$. 
     As expected the model on the left with scalar curvature $\approx 430$ is more robust to perturbations compared to the right model with scalar curvature $\approx 610$.}
     \label{fig:robutness_perturbations}
\end{SCfigure}

Let us assume two minima $\xx_{1}$ and $\xx_{2}$, and we suppose that the loss function at $\xx_{1}$ is flatter than the one at $\xx_{2}$ in terms of scalar curvature so $0 \leq \Sc(\xx_{1}) \leq \Sc(\xx_{2})$. Being at the minimum implies that $\Sc(\xx_{1}) = \trace(\Hess_1)^2 - \trace(\Hess^2_1)$ and $\Sc(\xx_{2}) = \trace(\Hess_2)^2 - \trace(\Hess^2_2)$ respectively. 
Then: 
\begin{align}
0 \leq \Sc(x_{1}) \leq \Sc(x_{2})    & \iff  0 \leq \trace(\Hess_1)^2 - \trace(\Hess^2_1) \leq  \trace(\Hess_2)^2 - \trace(\Hess^2_2)  \Rightarrow \trace(\Hess^2_{1}) \leq \trace(\Hess^2_{2}).
\end{align}
A flatter minima $\Sc(\xx_1)\leq \Sc(\xx_2)$ leads to more robustness of the loss function to weights perturbations: $\norm{ f(\xx_{1} + \varepsilon \xx) - f(\xx_{1}) }_{2}^2 \leq \norm{f(\xx_{2} + \varepsilon \xx) - f(\xx_{2}) }_{2}^2$. \\
In Figure~\ref{fig:robutness_perturbations}, we consider $\varepsilon\sim\mathcal{N}(0,0.01)$ to be a small perturbation and we plotted the original loss function with the perturbed losses. We computed the $\trace{\Hess^2}$ at the minimum. When the scalar curvature is smaller, the variance across the perturbations at the minimum is smaller and the perturbations are more centered around the original loss function.

\subsubsection*{Efficiency of escaping minima}
Stochastic gradient descent can be conceptualized as an Ornstein-Uhlenbeck process \parencite{uhlenbeck1930theory}, which is a continuous-time stochastic process that characterizes the behavior of a particle influenced by random fluctuations \parencite{mandt2017stochastic}. By considering the non-linear relationship between the weights and the covariance, the update rules in gradient descent resemble the optimization approach employed in the multivariate Ornstein-Uhlenbeck process. When approximating the covariance by the Hessian \parencite[Appendix A]{jastrzebski2017three}, the gradient descent can be seen as an Ornstein-Uhlenbeck process with: 
\begin{equation} \label{eq:MOUprocess}
    \text{d}\x_t = - \Hess \x_t \text{d}t + \Hess^{\frac{1}{2}} \text{d} W_t
\end{equation}

The escaping efficiency measure is a metric used to evaluate the performance of optimization algorithms, including gradient descent, in escaping from local minima and finding the global minimum of the loss function, and is defined as $\EE[ f(\x_t)- f(\xmin)]$. \textcite{zhu2018anisotropic} used this definition and the expression of the gradient descent process (Equation~\ref{eq:MOUprocess})  to approximate the escaping efficiency: 
\begin{equation}
    \EE[ f(\x_t)- f(\xmin)] \approx \frac{t}{2} \trace(\Hess^2).
\end{equation}

Similar to the example above, gradient descent will have more difficulties to escape from a minima with a small scalar curvature, and so it will converge more quickly to the flat minima. 

\subsubsection*{The scalar curvature is the squared norm of the Hessian in over-parameterized neural networks}

\begin{proposition} \label{prop:limit_scalar}
We note $\Hess$ the Hessian of the loss of a model with $q$ parameters, and $\Sc$ the scalar curvature, obtained in Proposition~\ref{proposition:full_scalar_curvature} and Corollary~\ref{cor:scalar_curvature}. When we reach a flat minimum, supposing the eigenvalues of $\Hess$ are similar, for a high number of parameters $q$, we have:
    \[ \Sc(\xx_{\text{min}}) \underset{q \to \infty}{\sim}  \trace(\Hess)^2 \]
\end{proposition}
\begin{proof}
Let us suppose that, at a flat minimum, all the eigenvalues are similar: $\lambda_1 = \cdots =\lambda_q = \lambda \geq 0$. Then we, have $\norm{\Hess}_{*}^2 = q^2 \lambda^2$ and $\norm{\Hess}_{F}^2 = q \lambda^2$. When the number of parameters increases, $\norm{\Hess}_{F}^2 = o(\norm{\Hess}_{*}^2)$, and as a consequence $\norm{\Hess}_{*}^2 - \norm{\Hess}_F^2 \sim \norm{\Hess}_{*}^2$.
\end{proof}

In this proposition, we assume that all the eigenvalues are similar. This strong assumption is supported by empirical results \parencite{ghorbani2019investigation}. The empirical results show that during the optimization process, the spectrum of the eigenvalues becomes entirely flat, especially when the neural network includes batch normalization.

\subsection{Reparametrization of the parameter space}

The main argument challenging the link between flatness and generalization is that the flatness definitions, so far, are not \textit{invariant under reparametrization}. Reparametrization refers to a change in the parametrization of the model, which can be achieved by transforming the original parameters ($\theta$) into a new set of parameters ($\eta$). Even if we assume that the models have the same performance: $\{f_{\theta}, \theta\in\Theta\subset\RR^q\} = \{f_{\varphi(\eta)}, \eta\in \varphi^{-1}(\Theta)\}$, this reparametrization alters the shape of the loss function landscape in $\RR^q$. This is the core of the problem: \textcite{dinh2017sharp} compared the flatness of $f_{\theta}$ and $f_{\varphi(\eta)}$ with respect to the same ambient space $\RR^q$, while each measure should be defined, and compared, relative to their respective parameter space, and not to an arbitrary space of reference. 

The scalar curvature is not invariant under reparametrization of the parameter space, and it should not be. It is, however, an \textbf{intrinsic} quantity, which means that it does not depend on an ambient space. As a consequence, it is also equivariant under diffeomorphism, and notably, if $\MM$ and $\MM'$ are two Riemannian manifolds related by an isometry $\Psi:\MM\to\MM'$, then $\Sc(\x) = \Sc(\Psi(\x))$, for all $\x\in\MM$.

In the case of the scalar curvature, if we apply a diffeomorphism to the parameters space with $\varphi:\MM\to\MM'$, and $f:\MM'\subset\RR^q\to \RR$ the loss function, then: 
\[\Hess(f \circ \varphi) = \J(\varphi)^{\top} \Hess(f) \J(\varphi) + \J_k(f) \Hess^k(\varphi),\]
with $\J(\varphi)$, $\J(f)$ the Jacobian of $\varphi$ and $f$, and $\Hess(f \circ \varphi)$, $\Hess(f)$ and $\Hess^k(\varphi)$ the Hessian of $f \circ \varphi$ and $f$. we note $\Hess^k(\varphi)_{ij} = \partial_i \partial_j \varphi^k$ the Hessian of the k-th component of $\varphi$.

At the minimum of the loss function, $\J(f)=0$, with $\varphi:\MM\to\MM'$ a diffeomorphism, and $\x'=\varphi(\x)$, the scalar curvatures on $\MM$ and $\MM'$ is derived as: 
\begin{align*}
\Sc(\x) &= \norm{\Hess_f}_{*}^2 - \norm{\Hess_f}_{F}^2,  \\
\Sc(\x') &= \norm{\J_{\varphi}\J_{\varphi}^{\top}\Hess_f}_{*}^2 - \norm{\J_{\varphi}\J_{\varphi}^{\top}\Hess_f}_{F}^2.
\end{align*}

\section{Discussion}

Our research focused on analyzing the loss landscape  as a Riemannian manifold and its connection to optimization generalization. We introduced a Riemannian metric on the parameter space and examined the scalar curvatures of the loss landscape. We found that the scalar curvature at minima is defined as the difference between the nuclear and Frobenius norm of the Hessian of the loss function. 

The \textit{flatness hypothesis} forms the basis of our study, suggesting that flat minima lead to better generalization compared to sharp ones. The Hessian of the loss function is known to be crucial in understanding optimization. However, analyzing the spectrum of the Hessian, particularly in over-parameterized models, can be challenging. As a result, the research community has started relying on the norm of the Hessian. We show that, in certain scenarios, the Hessian norm doesn't effectively gauge flatness, whereas scalar curvature does. Despite this, the Hessian norm is still relevant to theoretical results in optimization, including the model's stability against perturbations and the algorithm's ability to converge. Similarly, these characteristics are also satisfied by the scalar curvature. In essence, the scalar curvature combines all the advantages of the Hessian norm while accurately describing the curvature of the parameter space.

Future research could explore the curvature within stochastic optimization and investigate the scalar curvature as a random variable affected by the underlying data and batch distribution. It would also be interesting to understand how the scalar curvature relates to the stochastic process and whether it is connected to any implicit regularization in the model.

Overall, our study contributes to the understanding of the loss function's parameter space as a Riemannian manifold and provides insights into the curvature properties that impact optimization and generalization.

\newpage
\printbibliography
\newpage
\section*{Appendix}
\appendix
\input{appendix}

\end{document}

%% file: appendix.tex
\section{A primer on curvatures in Riemannian geometry} \label{appendix:definition}

The key strength of the Riemannian geometry is to allow for calculations to be conducted independently of the choice of the coordinates. However, this flexibility results in more sophisticated computations. Specifically, as a vector moves across a manifold, its local coordinates also change. We must consider this shift, which is accomplished by including a correction factor, denoted as $\Gamma$, to the derivative of the vector. These factors $\Gamma$ are known as \textbf{Christoffel symbols}.

\begin{definition}[\textbf{Christoffel symbols}] \label{def:christoffel_symbols}
    Let $(\MM,g)$ be a Riemannian manifold, and $\uu$ and $\vv$ two vector fields on $\MM$. On the manifold, we need to add the \textbf{Christoffel symbols} $\Gamma ^{k}_{ij}$ to account for the variation of the local basis represented by $\e_i$. The covariant derivative, or \textbf{connection}, is then defined by:
    \[\nabla _{\uu}{\vv} =u^i \partial_i v^j \e_j + u^i v^j \Gamma^{k}_{ji} \e_k, \]
    with $\nabla_{\uu}\vv = u^i \partial_i v^j \e_j$ the covariant derivative of $\vv$ along $\uu$ in the Euclidean plane. We can further compute the Christoffel symbols based on the Riemannian metric tensor $g_{ij}$:
    \[\Gamma^{k}_{ij} = \frac{1}{2} g^{kl} \left( \partial_i g_{jl} + \partial_j g_{il} - \partial_l g_{ij} \right),\]
\end{definition}

Now, we are interested in the concept of curvature. In Riemannian geometry, the curvature is defined as the deviation of the manifold from the Euclidean plane. The principal intrinsic tool that assess the curvature of a manifold is the \textbf{Riemann curvature tensor}, denoted $\R$. It characterises the change of the direction of a vector, when transported along an infinitesimally small closed loop. The Riemannian curvature tensor is defined the following way: 

\begin{definition}[\textbf{Riemann curvature tensor}] \label{def:riemann_curv}
    Let $(\MM, g, \nabla)$ be a Riemannian manifold. The \textbf{Riemannian curvature tensor} is defined by:
    \[\R(\xx, \yy; \zz) = \nabla_{\xx} \nabla_{\yy} \zz - \nabla_{\zz} \nabla_{\yy} \zz - \nabla_{[\xx,\yy]} \zz,\]
    for any vector fields $\xx, \yy, \zz \ \in\vecField$, with $[\cdot,\cdot]$ the Lie bracket. At the local basis represented by $\e_i$, it can be expressed in terms of indices: $\R^{l}_{ijk} = \e^l \R(\e_j, \e_k; \e_i)$, and in terms of the Christoffel symbols as:
    \[R_{ijk}^{l} = \partial_{i} \christoffel{l}{jk} - \partial_j \christoffel{l}{ik} + \christoffel{m}{jk}\christoffel{l}{im} - \christoffel{m}{ik}\christoffel{l}{jm} \]
\end{definition}

The Riemann curvature tensor being a fourth order tensor, it can difficult to interpret. Instead, we can look at a scalar quantity called the \textbf{scalar curvature} or equivalently the scalar Ricci curvature, which is a contraction of the Riemann curvature tensor.

\begin{definition}[\textbf{Scalar curvature}] \label{def:scalar_curv}
    Let $(\MM,g)$ be a Riemannian manifold. The \textbf{scalar curvature} is defined as:
    \[\Sc = g^{ij} \R^{k}_{ikj},\]
    using the Einstein summation convention, with $g^{ij}$ the inverse of the metric tensor $g_{ij}$, and $\R^{k}_{ikj}$ the components of the Riemannian curvature tensor.
\end{definition}

Just like the Riemannian curvature tensor and the Riemannian metric tensor, the scalar curvature is defined for every point on the manifold. The scalar curvature is null when the manifold is isometric to the Euclidean plane. It is be negative when the manifold is hyperbolic, or positive when the manifold is spherical. \\

\begin{figure}
    \centering
    \includegraphics[width=\textwidth]{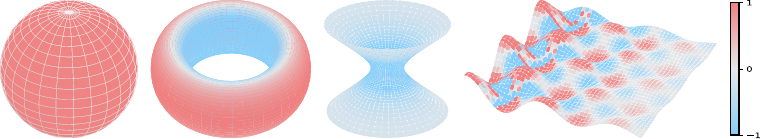}
    \caption{The scalar curvature is plotted for different surfaces in $\RR^3$. The sphere has always a positive scalar curvature since it is convex, the hyperboloid is always negative. In those figures, the values of the scalar curvature have been normalised between $-1$ and $1$.}
    \label{fig:scalar_surfaces}
\end{figure}

By definition, the scalar curvature is an intrinsic quantity, meaning that it does not depend on the ambient space. As a consequence, the scalar curvature is equivariant under diffeomorphisms. If we map a manifold $(\MM, g)$ to another manifold $(\MM', g', \nabla')$ with a diffeomorphism $\varphi: \MM' \to \MM$, we can express the connection $\nabla'$ as the pullback of $\nabla$: $\nabla' = d\varphi^{*}\nabla$. The curvature of the pullback connection is the pullback of the curvature of the original connection. In other terms: $d\varphi^{*} \Sc(\nabla) = \Sc(d\varphi^{*} \nabla)$ \parencite[Proposition 2.59]{andrews2010ricci}. In particular, if $\varphi$ is an isometry: $\Sc(\nabla) = \Sc(\nabla')$.

\section{Theoretical results} \label{appendix:proofs}

\subsection{Definition of the scalar curvature and other curvature measures}

\begin{proposition}
The Christoffel symbols of the metric $\G = \Id + \nabla_{x} f \nabla_{x} f^{\top}$, in the parameter space $\Omega\subset\RR^q$ with $ f$ the loss function is given by:  

\[\Gamma^{i}_{kl} = \frac{ f_{,i} f_{,kl}}{1+\norm{\nabla f}^2}\]
\end{proposition}

\begin{proof}

We use below the Einstein sum notation, and in particular, for the scalar function $ f$: $\partial_i \partial_j  f =  f_{,ij}$. 
The Christoffels symbols are obtained with the Riemannian metric: 
\[\Gamma_{kl}^i = \frac{1}{2} g^{im} \left( g_{mk,l} + g_{ml,k} - g_{kl,m}\right)\]

Our metric is $\G = \Id + \nabla f\ \nabla f^{\top}$. Using the Sherman-Morrison formula: $\G^{-1} = \Id - \frac{\nabla f\ \nabla f^{\top}}{1+ \norm{\nabla f}^2} $ 

\begin{align*}
g_{ij}    & = \G_{ij} = \delta_{ij} +  f_{,i}  f_{,j} \\
g_{ij, k} & =  f_{,ik}  f_{,j} +  f_{,i}  f_{,jk} \\
g_{mk, l} + g_{ml, k} - g_{kl, m} & = 2  f_{,kl}  f_{,m} \\
g^{im}    & = \G^{-1}_{im} = \delta_{im} - \frac{ f_{,i}  f_{,m}}{1+\norm{\nabla  f}^2} 
\end{align*}

Then: 
\[\Gamma^{i}_{kl} = \left(\delta_{im} - \frac{ f_{,i}  f_{,m}}{1+\norm{\nabla  f}^2} \right)  f_{,kl}  f_{,m} =  f_{,kl}  f_{,i} - \frac{  f_{,kl}  f_{,i}  f_{,m}^2}{1+\norm{\nabla  f}^2} = \frac{ f_{,i} f_{,kl}}{1+\norm{\nabla f}^2}. \]
\end{proof}

\begin{proof}
The coordinates of the Riemannian tensor curvature can be written with the Christoffel symbols:
\[R^\sigma_{\mu\nu\kappa} = \frac{\partial{\Gamma^\sigma}_{\mu\kappa}}{\partial x^\nu} - \frac{\partial{\Gamma^\sigma}_{\mu\nu}}{ \partial x^\kappa} + \Gamma^\sigma_{\nu\lambda} \Gamma^\lambda_{\mu\kappa} - \Gamma^\sigma_{\kappa\lambda} \Gamma^\lambda _{\mu\nu}\]
\end{proof}

\begin{lemma}
The metric tensor $\G = \Id + \nabla f \ \nabla f^{\top}$ has for eigenvalues: $\{1,1,\cdots, 1, 1+\norm{\nabla f}^2\}$.
\end{lemma}
\begin{proof}
$\G$ is a symmetric positive definite matrix, hence it is diagonalisable and all its eigenvectors $\vec{w}$ are orthogonal. Let's note $\vv = \nabla f$. For the eigenvector $\vv$: $\G\vv = (1+\norm{\vv}^2)\vv$. For all the other eigenvectors, $\inner{\vec{w}}{\vv} =0$ and $\G\vec{w}=\vec{w}$.
\end{proof}

\begin{proposition}
The contraction of the Christoffel symbols for the metric $\G = \Id + \nabla f \ \nabla f^{\top}$: 
\[\Gamma_{ki}^i = \frac{ f_{,ik} f_{,i}}{1+\norm{\nabla  f}^2}.\]
\end{proposition}

\begin{proof}
By definition, we have $\Gamma^i_{ki} = \partial_k \ln \sqrt{\det \G}$. By the previous lemma, we know that $\det G = 1+\norm{\nabla f}^2 = 1+ f_{,i}^2$.
\[
\Gamma^{i}_{ki} =  \partial_k \ln \sqrt{\det \G} =  \partial_k \ln \sqrt{ 1+{ f_{,i}}^2} = \frac{1}{2} \frac{\partial_k (1+ f_{,i}^2)}{1+\norm{\nabla f}^2} = \frac{ f_{,ik} f_{,i}}{1+\norm{\nabla f}^2}.
\]

Another method is to use the general expression of $\Gamma^{i}_{kl} = \frac{ f_{,i} f_{,kl}}{1+\norm{\nabla f}^2}$, and the result is obtained for $i=l$.
\end{proof}

\begin{proposition}
The Riemannian curvature tensor is given by: 
\[R^i_{jkm} = \beta ( f_{,ik} f_{,jm} -  f_{,jm} f_{,jk}) - \beta^2  f_{,i} f_{,r} ( f_{,rk} f_{,im} -  f_{,rm} f_{,jk}) \]
\end{proposition}

\begin{proof}

The Riemannian curvature tensor is given by: $R^i_{jkm} = \partial_k \Gamma^i_{jm} - \partial_m \Gamma^{i}_{jk} + \Gamma^{i}_{rk}\Gamma^r_{jm} - \Gamma^{i}_{rm}\Gamma^r_{jk}$, and we have for Christoffel symbols: $\Gamma^{i}_{jm} = \beta  f_{,i} f_{,jm}$. \\

We note $\beta = (1+\norm{\nabla f}^2)^{-1}$. 
We have: $\partial_k (\beta  f_{,i} f_{,jm}) = \partial_k (\beta)  f_{,i} f_{,jm} + \beta ( f_{,ik} f_{,jm}+  f_{,i} f_{,jmk})$, and $\partial_k (\beta) = - 2 \beta^2  f_{ka}  f_{a}$.

\begin{align*}
\partial_k \Gamma^i_{jm}    & = -2\beta^2  f_{,a} f_{,ak} f_{,i} f_{,jm} + \beta ( f_{,ik} f_{,jm}+ f_{,i} f_{,jmk}) \\
\partial_m \Gamma^i_{jk}    & = -2\beta^2  f_{,a} f_{,ak} f_{,i} f_{,jm} + \beta  ( f_{,im} f_{,jk}+ f_{,i} f_{,jkm}) \\
\Gamma^{i}_{rk}\Gamma^r_{jm} & = \beta^2  f_{,i} f_{,rk} f_{,r} f_{,jm}\\
\Gamma^{i}_{rm}\Gamma^r_{jk} & = \beta^2  f_{,i} f_{,rm} f_{,r} f_{,jk}\\
\end{align*}

\end{proof}

\begin{proposition}
The Ricci scalar curvature is given by: 
\[R =  \beta \left(\trace(\Hess)^2 - \trace(\Hess^2)\right) + 2 \beta^2 \left( \nabla f^{\top} (\Hess^2 - \trace(\Hess) \Hess) \nabla f\right),\]

with $\Hess$ the Hessian of $ f$. 
\end{proposition}

\begin{proof}
We use $\beta^{-1} = 1+\norm{\nabla f}^2$, $\Hess$ the Hessian of $ f$, and $\norm{\cdot}_{1,1}$ the matrix norm $L_{1,1}$. \\

The Ricci tensor is given by: 
\begin{align*}
R_{ab} = R^i_{aib} & = \beta ( f_{,ii} f_{,ab}- f_{,bi} f_{,ai}) - \beta^2  f_{,i} f_{,r}( f_{,ir} f_{,ab}- f_{,br} f_{,ai}) \\
 & = \beta (\trace(\Hess)\Hess_{ab}-\Hess_{ab}^2) - \beta^2 \left((\nabla f^{\top}\Hess\nabla f)\Hess_{ab}-(\Hess\nabla f)_{a}(\Hess\nabla f)_{b}\right) \\
\end{align*}

The Ricci scalar is given by $g^{ab} R_{ab} = \delta_{ab} R_{ab} - \beta  f_{,a} f_{,b} R_{ab}$, and we notice:

\begin{align*}
\Hess_{aa} & = \trace(\Hess)  \\
 f_{,a}\Hess_{ab} f_{b} & = \nabla f^{\top}\Hess\nabla f \\
(\Hess\nabla f)_{a} f_{,a} & = \nabla f^{\top}\Hess\nabla f \\
\end{align*}

\begin{align*}
R_{ab} & = R_{aa} - \beta  f_{,a} f_{,b} R_{ab} \\
R_{aa} & = \beta (\trace(\Hess)^2 - \trace(\Hess)^2) - \beta^2 \left((\nabla f^{\top}\Hess\nabla f)\trace(\Hess) - \nabla f^{\top}\Hess^2\nabla f\right) \\
 \beta  f_{,a} f_{,b} R_{ab} & = \beta^2 \left(\nabla f^{\top}\Hess\nabla f)\trace(\Hess) - \nabla f^{\top}\Hess^2\nabla f\right) - \beta^3 \left((\nabla f^{\top}\Hess\nabla f)^2 - (\nabla f^{\top}\Hess\nabla f)^2\right)\\
\end{align*}

Then: 
\[R =  \beta \left(\trace(\Hess)^2 - \trace(\Hess^2)\right) - 2 \beta^2 \left( \nabla f^{\top} (\trace(\Hess) \Hess - \Hess^2) \nabla f\right)\]
\end{proof}

\subsection{Perturbations on the weights}
\begin{proposition}
Let $\xx_{\text{min}}$ an extremum, $\varepsilon \ll 1$ and $\x$ a normalized vector. Then, minimising the trace of the square of the Hessian is equivalent to minimising the influence of the perturbations on the weights: 
\begin{equation}
    \norm{f(\xx_{\text{min}} + \varepsilon \xx) - f(\xx_{\text{min}})}_{2}^2 \leq \frac{1}{4} \varepsilon^4 \trace(\Hess^2_{\text{min}})
\end{equation}
\end{proposition}

\begin{proof}
    The general Taylor expansion on $f$ at $\xx_{\text{min}} + \varepsilon \xx$, with $\varepsilon \ll 1$ is: 
    \[f(\xx_{\text{min}} + \varepsilon\xx) = f(\xx_{\text{min}}) + \varepsilon\xx^{\top}\J + \frac{\varepsilon^2}{2} \xx^{\top} \Hess \xx + o(\varepsilon^2 \norm{\xx}^2). \]

    We now assume that $\xx$ is normalised such that $\norm{\xx}=1$. Note that, if $\xx$ is an eigenvector of $\Hess$ then: $\xx^{\top} \Hess \xx = \trace(\Hess)$. In general, each element of the vector is inferior to 1: $\xx_i^2 \leq 1$ and so, $\lambda_i^2 \xx_i^4 \leq \lambda_i^2$. Furthermore, we have $\J(\xx_{\text{min}}) = 0$. Thus: 

    \[ \norm{f(\xx_{\text{min}} + \xx) - f(\xx_{\text{min}})}_2^2 = \frac{\varepsilon^4}{4} \left(\xx^{\top}\Hess \xx\right)^2 + o(\varepsilon^4) \leq \frac{\varepsilon^4}{4} \trace(\Hess^2)  + o(\varepsilon^4) \]

\end{proof}

\subsection{Curvature over minibatches}

\begin{proposition}
    The Scalar curvature of the hessian of the full dataset is not equal to the expectation of the Scalar curvature over mini-batches. That is there exists a dataset, $\mathcal{D}$, and mini-batches, $\{ \mathcal{B}_1, \mathcal{B}_2, \dots, \mathcal{B}_k \}$ such that:
    \[R(\Hess_{\mathcal{D}}) \neq \mathbb{E}[R(\Hess_{\mathcal{B}_i})] \]
\end{proposition}

\begin{proof}
    Suppose we have a dataset $\mathcal{D}$ and mini-batches $\{ \mathcal{B}_1, \mathcal{B}_2\}$ such that the Hessians over the minibatches are given by: 
    \[
    \begin{bmatrix}
    -2 & 0\\
    4 & 1
    \end{bmatrix}
    \text{,}
    \begin{bmatrix}
    1 & 2 & 1\\
    2 & -2
    \end{bmatrix}
    \]

    They both have equal trace, $-1$, and their ricci curvatures are $-2$ and $-6$ respectively. The hessian over the full dataset is given by: 
    \[
    \begin{bmatrix}
    -1 & 2\\
    6 & -1
    \end{bmatrix}
    \]
    This has the same trace as the minibatches but its ricci curvature is $-22$ not equal to the average of the ricci curatures over miniabtches. 
    
\end{proof}

%% file: main.bib
@inproceedings{dinh2017sharp,
  title={Sharp minima can generalize for deep nets},
  author={Dinh, Laurent and Pascanu, Razvan and Bengio, Samy and Bengio, Yoshua},
  booktitle={International Conference on Machine Learning},
  pages={1019--1028},
  year={2017},
  organization={PMLR}
}

@book{andrews2010ricci,
  title={The Ricci flow in Riemannian geometry: a complete proof of the differentiable 1/4-pinching sphere theorem},
  author={Andrews, Ben and Hopper, Christopher},
  year={2010},
  publisher={springer}
}

@article{zhu2018anisotropic,
  title={The anisotropic noise in stochastic gradient descent: Its behavior of escaping from sharp minima and regularization effects},
  author={Zhu, Zhanxing and Wu, Jingfeng and Yu, Bing and Wu, Lei and Ma, Jinwen},
  journal={arXiv preprint arXiv:1803.00195},
  year={2018}
}

@article{jastrzebski2017three,
  title={Three factors influencing minima in sgd},
  author={Jastrzebski, Stanislaw and Kenton, Zachary and Arpit, Devansh and Ballas, Nicolas and Fischer, Asja and Bengio, Yoshua and Storkey, Amos},
  journal={arXiv preprint arXiv:1711.04623},
  year={2017}
}

@inproceedings{orvieto2022anticorrelated,
  title={Anticorrelated noise injection for improved generalization},
  author={Orvieto, Antonio and Kersting, Hans and Proske, Frank and Bach, Francis and Lucchi, Aurelien},
  booktitle={International Conference on Machine Learning},
  pages={17094--17116},
  year={2022},
  organization={PMLR}
}

@article{hochreiter1997flat,
  title={Flat minima},
  author={Hochreiter, Sepp and Schmidhuber, J{\"u}rgen},
  journal={Neural computation},
  volume={9},
  number={1},
  pages={1--42},
  year={1997},
  publisher={MIT Press One Rogers Street, Cambridge, MA 02142-1209, USA journals-info~…}
}

@article{keskar2016large,
  title={On large-batch training for deep learning: Generalization gap and sharp minima},
  author={Keskar, Nitish Shirish and Mudigere, Dheevatsa and Nocedal, Jorge and Smelyanskiy, Mikhail and Tang, Ping Tak Peter},
  journal={arXiv preprint arXiv:1609.04836},
  year={2016}
}

@article{chaudhari2019entropy,
  title={Entropy-sgd: Biasing gradient descent into wide valleys},
  author={Chaudhari, Pratik and Choromanska, Anna and Soatto, Stefano and LeCun, Yann and Baldassi, Carlo and Borgs, Christian and Chayes, Jennifer and Sagun, Levent and Zecchina, Riccardo},
  journal={Journal of Statistical Mechanics: Theory and Experiment},
  volume={2019},
  number={12},
  pages={124018},
  year={2019},
  publisher={IOP Publishing}
}

@inproceedings{hinton1993keeping,
  title={Keeping the neural networks simple by minimizing the description length of the weights},
  author={Hinton, Geoffrey E and Van Camp, Drew},
  booktitle={Proceedings of the sixth annual conference on Computational learning theory},
  pages={5--13},
  year={1993}
}

@article{buntine1991bayesian,
  title={Bayesian backpropagation},
  author={Buntine, Wray L},
  journal={Complex systems},
  volume={5},
  pages={603--643},
  year={1991}
}

@article{pittorino2021entropic,
  title={Entropic gradient descent algorithms and wide flat minima},
  author={Pittorino, Fabrizio and Lucibello, Carlo and Feinauer, Christoph and Perugini, Gabriele and Baldassi, Carlo and Demyanenko, Elizaveta and Zecchina, Riccardo},
  journal={Journal of Statistical Mechanics: Theory and Experiment},
  volume={2021},
  number={12},
  pages={124015},
  year={2021},
  publisher={IOP Publishing}
}

@article{yi2019positively,
  title={Positively scale-invariant flatness of relu neural networks},
  author={Yi, Mingyang and Meng, Qi and Chen, Wei and Ma, Zhi-ming and Liu, Tie-Yan},
  journal={arXiv preprint arXiv:1903.02237},
  year={2019}
}

@article{neyshabur2017geometry,
  title={Geometry of optimization and implicit regularization in deep learning},
  author={Neyshabur, Behnam and Tomioka, Ryota and Salakhutdinov, Ruslan and Srebro, Nathan},
  journal={arXiv preprint arXiv:1705.03071},
  year={2017}
}

@article{izmailov2018averaging,
  title={Averaging weights leads to wider optima and better generalization},
  author={Izmailov, Pavel and Podoprikhin, Dmitrii and Garipov, Timur and Vetrov, Dmitry and Wilson, Andrew Gordon},
  journal={arXiv preprint arXiv:1803.05407},
  year={2018}
}

@inproceedings{ju2023generalization,
  title={Generalization in Graph Neural Networks: Improved PAC-Bayesian Bounds on Graph Diffusion},
  author={Ju, Haotian and Li, Dongyue and Sharma, Aneesh and Zhang, Hongyang R},
  booktitle={International Conference on Artificial Intelligence and Statistics},
  pages={6314--6341},
  year={2023},
  organization={PMLR}
}

@article{imaizumi2022generalization,
  title={On generalization bounds for deep networks based on loss surface implicit regularization},
  author={Imaizumi, Masaaki and Schmidt-Hieber, Johannes},
  journal={IEEE Transactions on Information Theory},
  volume={69},
  number={2},
  pages={1203--1223},
  year={2022},
  publisher={IEEE}
}

@book{gallot1990riemannian,
  title={Riemannian geometry},
  author={Gallot, Sylvestre and Hulin, Dominique and Lafontaine, Jacques and others},
  volume={2},
  year={1990},
  publisher={Springer}
}

@article{wei2019noise,
  title={How noise affects the hessian spectrum in overparameterized neural networks},
  author={Wei, Mingwei and Schwab, David J},
  journal={arXiv preprint arXiv:1910.00195},
  year={2019}
}

@inproceedings{ghorbani2019investigation,
  title={An investigation into neural net optimization via hessian eigenvalue density},
  author={Ghorbani, Behrooz and Krishnan, Shankar and Xiao, Ying},
  booktitle={International Conference on Machine Learning},
  pages={2232--2241},
  year={2019},
  organization={PMLR}
}

@book{lee2018introduction,
  title={Introduction to Riemannian manifolds},
  author={Lee, John M},
  volume={2},
  year={2018},
  publisher={Springer}
}

@inproceedings{seong2018towards,
  title={Towards Flatter Loss Surface via Nonmonotonic Learning Rate Scheduling.},
  author={Seong, Sihyeon and Lee, Yegang and Kee, Youngwook and Han, Dongyoon and Kim, Junmo},
  booktitle={UAI},
  pages={1020--1030},
  year={2018}
}

@article{uhlenbeck1930theory,
  title={On the theory of the Brownian motion},
  author={Uhlenbeck, George E and Ornstein, Leonard S},
  journal={Physical review},
  volume={36},
  number={5},
  pages={823},
  year={1930},
  publisher={APS}
}

@article{mandt2017stochastic,
  title={Stochastic gradient descent as approximate bayesian inference},
  author={Mandt, Stephan and Hoffman, Matthew D and Blei, David M},
  journal={arXiv preprint arXiv:1704.04289},
  year={2017}
}
